\documentclass[runningheads]{llncs}
\usepackage[T1]{fontenc}
\usepackage{hyperref}
\usepackage{url}

\usepackage[utf8]{inputenc} 
\usepackage{booktabs,enumitem}       
\usepackage{amsfonts}       
\usepackage{nicefrac}       
\usepackage{microtype}      
\usepackage{xcolor,colortbl}         
\usepackage{algorithm}
\usepackage{algpseudocode}

\hypersetup{
    colorlinks,
    citecolor=blue,
    linkcolor=blue,
    filecolor=blue,      
    urlcolor=magenta,
}

\newcommand{\partitle}[1]{\smallskip \noindent \textbf{#1.}}
\usepackage{subfigure}
\usepackage{graphicx}
\usepackage{amsmath}
\usepackage{multirow}
\usepackage{makecell}
\usepackage{amssymb}
\usepackage{wrapfig}
\usepackage{tcolorbox}
\tcbuselibrary{skins, breakable}
\usepackage{fontawesome}

\begin{document}
\title{Lightweight Time Series Data Valuation on Time Series Foundation Models via In-Context Finetuning}

\titlerunning{Time Series Data Valuation on TSFMs via In-context Finetuning}

\author{
Shunyu Wu\inst{1}\inst{2}\thanks{These authors contributed equally to this work. \textsuperscript{\faEnvelope} Dan Li is the corresponding author.} \and
Tianyue Li\inst{1}$^{\star}$ \and
Yixuan Leng\inst{1} \and
Jingyi Suo\inst{1} \and \\
Jian Lou\inst{1}\inst{2} \and
Dan Li\inst{1}\textsuperscript{\faEnvelope} \and
See-Kiong Ng\inst{3}
\email{\{wushy88,lity93,lengyx3,suojy3\}@mail2.sysu.edu.cn,
       \{louj5,lidan263\}@mail.sysu.edu.cn} \quad \email{seekiong@nus.edu.sg}
\institute{
Sun Yat-sen University \\
\and
State Key Laboratory of Internet
Architecture, Tsinghua University, Beijing, China, 100084
\and
National University of Singapore
}
}

\authorrunning{Wu et al.}

\maketitle

\begin{abstract}
Time series foundation models (TSFMs) have demonstrated increasing capabilities due to their extensive pretraining on large volumes of diverse time series data. Consequently, the quality of time series data is crucial to TSFM performance, rendering an accurate and efficient data valuation of time series for TSFMs indispensable. However, traditional data valuation methods, such as influence functions, face severe computational bottlenecks due to their poor scalability with growing TSFM model sizes and often fail to preserve temporal dependencies.
In this paper, we propose \textbf{LTSV}, a \textbf{L}ightweight \textbf{T}ime \textbf{S}eries \textbf{V}aluation on TSFMS via in-context finetuning. Grounded in the theoretical evidence that in-context finetuning approximates the influence function, LTSV estimates a sample's contribution by measuring the change in context loss after in-context finetuning, leveraging the strong generalization capabilities of TSFMs to produce robust and transferable data valuations. To capture temporal dependencies, we introduce temporal block aggregation, which integrates per-block influence scores across overlapping time windows. Experiments across multiple time series datasets and models demonstrate that LTSV consistently provides reliable and strong valuation performance, while maintaining manageable computational requirements. Our results suggest that in-context finetuning on time series foundation models provides a practical and effective bridge between data attribution and model generalization in time series learning. The code for LTSV is available at \url{https://github.com/b1302550313/LTSV}.

\renewcommand{\thefootnote}{\fnsymbol{footnote}}
\footnotetext[2]{This paper has been accepted as a full paper in DASFAA 2026.}
\renewcommand{\thefootnote}{\arabic{footnote}}

\keywords{Time Series Data Valuation \and Time Series Foundation Model.}
\end{abstract}

\section{Introduction}
\label{sec:intro}

Recent advances in time series foundation models (TSFMs)~\cite{goswami2024moment,rasul2023lag,liutimer2024timer,woo2024unified,liu2025sundial,arango2025chronosx}, built upon large and diverse temporal corpora, have greatly advanced generalizable and adaptive modeling across domains such as finance, healthcare, and climate science. However, the performance of these large models fundamentally depends on the quality of their training data~\cite{wangoptimal,fang2023bayotide}, making time series data valuation an indispensable component of TSFM development. Time series data valuation aims to quantify the contribution of individual time series samples to the overall model performance, enabling principled data selection, quality-aware training, and efficient resource allocation~\cite{sun20242d,nguyen2023context,zhang2025fairshare}. Accurate data valuation helps identify informative and trustworthy time series samples while filtering out corrupted or unrepresentative ones, ultimately improving both data efficiency and model performance~\cite{kwon2023data}. As TSFMs continue to expand in scale and scope, it becomes a compelling task to develop effective and scalable data valuation methods tailored to TSFMs.

Existing Time series data valuation methods, including TimeInf~\cite{zhang2024timeinf} and TimeShap~\cite{bento2021timeshap}, extend theoretically grounded influence functions~\cite{hampel1974influence} and Shapley-values~\cite{shapley1953value} to temporal settings. These methods can effectively capture the temporal dependence nature of time series data~\cite{kunsch1984infinitesimal} by incorporating temporal segmentation and sequence-aware analysis. Despite their effectiveness and theoretical foundations, these methods unfortunately inherit the computational burdens of their underlying theory: influence function-based approaches require costly Hessian and gradient calculations, while Shapley-based methods incur exponential subset sampling. Consequently, although effective for small-scale linear or recurrent models, these approaches are infeasible for large-scale TSFMs that typically come with millions or billions of parameters.

To address these challenges, we propose LTSV, a \textbf{L}ightweight \textbf{T}ime \textbf{S}eries \textbf{V}aluation on TSFMS via in-context finetuning designed to achieve high-fidelity data valuation. Our approach builds upon the classical influence function framework, adapting it to TSFMs to leverage their strong representational capacity and finetuning flexibility for estimating time series data values. 
To make this estimation lightweight, LTSV approximates the conventional influence function by performing one-step in-context finetuning, treating the context data as validation data to compute the change in context loss, and the target data as training data in the influence function framework. In essence, LTSV requires only a single gradient update per sample, therefore avoiding the computational burden caused by the Hessian-related terms in influence functions. It can be interpreted as a localized influence approximation, capturing how a single time series data point perturbs the model within its immediate optimization neighborhood. Consequently, this formulation significantly reduces computational complexity while maintaining the fidelity of estimating each sample’s contribution.
To further preserve temporal dependencies, LTSV incorporates temporal block aggregation, dividing sequences into overlapping time blocks and integrating data values across them to capture inter-temporal relationships. Extensive experiment evaluations demonstrate that LTSV not only provides faithful and effective valuations on the foundation models themselves but also exhibits strong generalization to downstream models, even enabling high-quality data identified on TSFMs to benefit diverse contentional time series models.

\partitle{Contributions} Our main contributions are as follows:
\begin{itemize}[leftmargin=*]
    \item We tackle the important and emerging problem of performing accurate and scalable data valuation on modern time series foundation models, for which existing methods are computationally prohibitive.
    \item We propose LTSV, a novel time series data valuation framework that repurposes in-context finetuning on TSFMs for lightweight time series data valuation, overcoming computational hurdles while preserving the fidelity of data value estimation.
    \item Extensive experiments on five popular datasets corroborate that LTSV provides reliable valuations on three representative TSFM model types and even generalizes effectively to diverse conventional time series models, enabling high-quality data to improve performance across diverse architectures while remaining computationally efficient.
\end{itemize}

\section{Related Work and Preliminaries}
\label{rw}

\subsection{Time Series Foundation Models}
Recent advancements in time series foundation models (TSFMs) have significantly reshaped the landscape of time series analysis and modeling. Similar to large language models (LLMs), TSFMs are pretrained on massive and heterogeneous time series corpora to learn universal temporal representations that can be adapted to various downstream tasks. These models adopt diverse architectures, including encoder-only Transformers~\cite{goswami2024moment,woo2024unified}, decoder-only autoregressive Transformers~\cite{garza2023timegpt,liutimer2024timer,das2024decoder}, and mixture-of-experts based sparse structures~\cite{xiaoming2025time,liu2024moirai,wu2025unlocking}. Model sizes now range from tens of millions to over 2.4 billion parameters~\cite{xiaoming2025time}, trained on tens of billions of temporal observations collected from diverse real-world domains.

While this scaling has led to increasingly powerful and versatile temporal representations, it also amplifies the dependency of TSFM performance on high-quality time series data. Low-quality or biased time series samples can degrade the learned representations and reduce downstream generalization ability~\cite{zhang2024irregular,liu2024time}.
Therefore, accurately valuing and selecting high-quality time series data for TSFMs becomes a fundamental yet underexplored problem, motivating the development of our proposed LTSV.

\subsection{Data Valuation Methods}
\partitle{Time Series Data Valuation} Time series data valuation requires specialized techniques that explicitly consider temporal dependencies and correlations. Building upon classical principles such as the influence function~\cite{hampel1974influence} and Shapley value~\cite{shapley1953value}, several methods have been proposed to extend data attribution into the time series domain. For example, TimeInf~\cite{zhang2024timeinf} adapts the influence function to temporal settings by modeling the effect of perturbing autoregressive processes, while TimeShap~\cite{bento2021timeshap} and WindowShap~\cite{nayebi2023windowshap} generalize Shapley-based attribution through sliding window decomposition and temporal feature aggregation. Subsequent studies extend influence estimation to multivariate settings~\cite{wang2024channel} and incorporate sequence-aware Shapley computations into Transformer architectures~\cite{cheng2025unifying}. Despite these advances, both influence-function and Shapley-based approaches remain computationally expensive, requiring either costly Hessian inversions or exponential subset sampling, and thus making them impractical for large-scale TSFMs with millions of parameters.

\partitle{Data Valuation for Foundation Models} Beyond time series–specific approaches, data valuation on foundation models has been widely explored in the NLP domain. Existing works employ strategies such as rule-based scoring~\cite{penedo2024fineweb}, semantic deduplication metrics~\cite{tirumala2023d4}, and proximity-based quality assessment via perplexity~\cite{du2022glam} or embedding similarity~\cite{xie2023data}, while recent advances leverage powerful reference LLMs for quality estimation~\cite{wettig2024qurating,sachdeva2024train,yu2024mates}.
However, these methods rely on text-centric characteristics, making them ill-suited for TSFMs characterized by non-stationary, correlated, and multi-channel dynamics. To address these gaps, our LTSV framework leverages the in-context fine-tuning to approximate influence-based valuation in a lightweight and temporally aware manner.


\subsection{Preliminaries}
\label{sec:prelim}

We revisit the classical Influence Function~\cite{koh2017understanding}, 
which serves as the theoretical basis for our proposed time series data valuation framework.
 
\begin{theorem}[Classical Influence Function]
\label{thm:influence}
Let the target dataset be $\mathcal{D}_{\mathrm{target}} = \{(x_i, y_i)\}_{i=1}^{N}$ and the context dataset be $\mathcal{D}_{\mathrm{context}} = \{(x'_j, y'_j)\}_{j=1}^{M}$. 
Let $\mathcal{L}(x, y; \theta)$ denote the loss function of a model parameterized by $\theta$ 
for a target sample. 
The influence function quantifies the effect of a small perturbation to a target sample $z = (x, y)$ on the context loss of $z' = (x', y')$ under the optimal parameters $\theta^*$, and is formally defined as
\begin{equation}
\label{eq:infl-ori}
\text{Infl}(z, z') 
= \nabla_\theta \mathcal{L}(x', y'; \theta^*)^\top \Delta \theta 
\approx 
- \nabla_\theta \mathcal{L}(x', y'; \theta^*)^\top 
H_{\theta^*}^{-1} 
\nabla_\theta \mathcal{L}(x, y; \theta^*),
\end{equation}
where $H_{\theta^*} = \frac{1}{N} \sum_{i=1}^{N} \nabla_\theta^2 \mathcal{L}(x_i, y_i; \theta^*)$ 
is the Hessian matrix of the empirical loss at $\theta^*$, 
and $\Delta \theta$ represents the parameter change resulting from an infinitesimal upweighting of $z$.
The approximation arises from a first-order Taylor expansion around $\theta^*$,
which neglects higher-order terms in $\Delta \theta$.
\end{theorem}

\begin{proof}
The derivation follows the first-order influence approximation in \cite{pruthi2020estimating,koh2017understanding}.
\end{proof}

This classical formulation provides a principled mechanism to quantify how a single target sample affects the model's prediction or context loss.
However, computing the inverse Hessian $H_{\theta^*}^{-1}$ is computationally infeasible for large-scale deep or foundation models~\cite{basu2020influence}, 
which motivates the development of more efficient data valuation approaches such as our proposed LTSV.

\section{The Proposed LTSV}

This section presents LTSV, a lightweight time series data valuation method for TSFMs via in-context finetuning, which circumvents the prohibitive computational cost of directly calculating influence functions while approximating them with theoretical support, thereby preserving high fidelity in TS data quality estimation.
We begin by reformulating classical influence functions through a first-order approximation that connects data influence to in-context finetuning (\S\ref{sec:in-context-approx}). 
We then extend this formulation to the time series domain, introducing a hierarchical valuation mechanism that operates at the block, point, and sample levels (\S\ref{sec:ts-adaptation}). 
Finally, we analyze the computational efficiency of the proposed method and demonstrate its scalability to modern TSFMs (\S\ref{sec:efficiency}). An overview of the entire framework is illustrated in Figure~\ref{fig:framework}.

\begin{figure*}[t]
    \centering
    \includegraphics[width=1.0\textwidth]{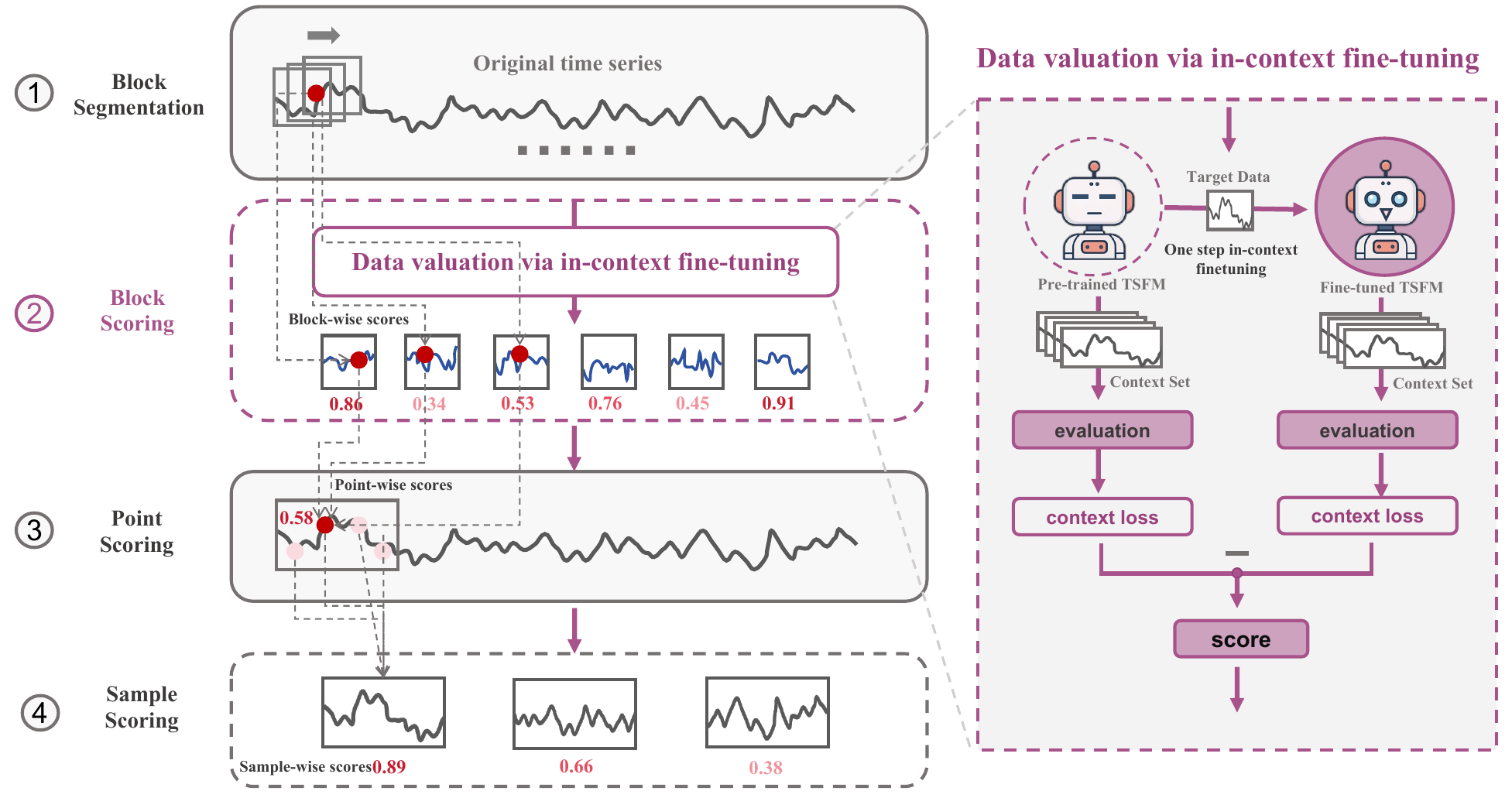}
    \vskip -0.5em
    \caption{Overview of the proposed framework. \textcircled{1} Block Segmentation: the original time series is first divided into sub-sequences via sliding block segmentation. \textcircled{2} Block Scoring: each block is applied to fine-tune the TSFM, and the differences between the pre-trained context losses and fine-tuned losses are calculated as the block-wise quality scores. \textcircled{3} Point Scoring: point-wise scores are aggregated based on the block-wise scores along the original series. \textcircled{4} Sample Scoring: sample-wise scores are generated based on the scores of each point.}
    \label{fig:framework}
    \vspace{-1em}
\end{figure*}

\subsection{Influence Function Approximation via In-Context Finetuning}
\label{sec:in-context-approx}


\begin{definition}[In-Context Finetuning]
\label{def:incontext}
In-context finetuning is a localized adaptation of a pretrained time series foundation model 
$\mathcal{M}_\theta$ that leverages the information in a \emph{target sample} 
$z = (x, y) \in \mathcal{D}_{\mathrm{target}}$ to improve the model's predictive performance 
on a \emph{context sample} $z' = (x', y') \in \mathcal{D}_{\mathrm{context}}$.

Formally, given the original model parameters $\theta$, a single-step in-context fine-tuning updates the model as
\begin{equation}
\theta_{\mathrm{finetuned}} = \theta - \eta \nabla_\theta \mathcal{L}(x, y; \theta),
\end{equation}
where $\eta$ is the learning rate and $\mathcal{L}(x, y; \theta)$ denotes the loss on the target sample $z$.

The effect of the target sample $z$ on the context sample $z'$ can be quantified by the change in context loss after in-context fine-tuning, i.e., 
$\mathcal{L}(x', y'; \theta) - \mathcal{L}(x', y'; \theta_{\mathrm{finetuned}})$.
This provides a simple and tractable approximation to the classical influence function.

\end{definition}

\partitle{Theoretical Evidence}
To provide theoretical support for the connection between in-context fine-tuning and influence functions, we show that the change in context loss induced by a single-step fine-tuning approximates the classical influence of a target sample on a context sample.

\begin{theorem}[Influence Function Approximation via In-Context Finetuning]
\label{thm:incontext}
Given a target sample $z = (x, y)$ and a context sample $z' = (x', y')$, 
the influence of $z$ on $z'$ can be locally approximated by the change in context loss 
after a single in-context fine-tuning step:
\begin{equation}
\label{eq:loss-diff}
\text{Infl}(z, z') \propto 
\mathcal{L}(x', y'; \theta) - \mathcal{L}(x', y'; \theta_{\mathrm{finetuned}}),
\end{equation}
where $\theta_{\mathrm{finetuned}} = \theta - \eta \nabla_\theta \mathcal{L}(x, y; \theta)$
is obtained by updating $\theta$ with one gradient step on the target sample $z$.
\end{theorem}

\begin{proof}
A single-step gradient update on the target sample $z$ gives $\theta_{\mathrm{finetuned}} = \theta - \eta \nabla_\theta \mathcal{L}(x, y; \theta)$,
where $\eta$ is the learning rate.
This first-order approximation is valid for a sufficiently small $\eta$~\cite{pruthi2020estimating}. Substituting $\Delta \theta = \theta_{\mathrm{finetuned}} - \theta \approx -\eta \nabla_\theta \mathcal{L}(x, y; \theta)$ into Eq.(\ref{eq:infl-ori}) yields
\begin{equation}
\text{Infl}(z, z') \approx -\eta \nabla_\theta \mathcal{L}(x', y'; \theta)^\top \nabla_\theta \mathcal{L}(x, y; \theta).
\label{eq:infl-grad-inner-target}
\end{equation}

To relate this gradient inner product to observable quantities, we perform a first-order Taylor expansion of the context loss around $\theta$:
\begin{equation}
\mathcal{L}(x', y'; \theta_{\mathrm{finetuned}}) 
= \mathcal{L}(x', y'; \theta) 
+ \nabla_\theta \mathcal{L}(x', y'; \theta)^\top (\theta_{\mathrm{finetuned}} - \theta) 
+ \mathcal{O}(\|\theta_{\mathrm{finetuned}} - \theta\|^2).
\end{equation}
Rearranging terms and substituting $\theta_{\mathrm{finetuned}} - \theta = -\eta \nabla_\theta \mathcal{L}(x, y; \theta)$, we obtain
\begin{equation}
\mathcal{L}(x', y'; \theta) - \mathcal{L}(x', y'; \theta_{\mathrm{finetuned}})
\approx \eta \nabla_\theta \mathcal{L}(x', y'; \theta)^\top \nabla_\theta \mathcal{L}(x, y; \theta).
\label{eq:loss-diff-target}
\end{equation}

Comparing Eq.\eqref{eq:infl-grad-inner-target} and Eq.\eqref{eq:loss-diff-target}, we obtain the formula stated in Theorem~\ref{thm:incontext}:
$\text{Infl}(z, z') \propto \mathcal{L}(x', y'; \theta) - \mathcal{L}(x', y'; \theta_{\mathrm{finetuned}})$.
The negative sign is absorbed into the context loss difference, so a positive value indicates that including $z$ reduces the loss on $z'$, reflecting a beneficial impact on model performance.
\end{proof}

Therefore, the influence of a target sample can be locally approximated by the loss reduction on context samples after one-step in-context finetuning~\cite{iter2023context,yu2024mates}. 
This equivalence provides theoretical evidence for approximating the computationally demanding influence function with more computationally efficient approximations via in-context finetuning and context loss calculation, forming the foundation of our lightweight time series data valuation methods for TSFMs presented below.

\subsection{Time Series Data Valuation on TSFM via In-Context Finetuning}
\label{sec:ts-adaptation}



Given a time series dataset $\mathcal{D} = \{ x_1, x_2, \dots, x_T \}$ consisting of $T$ sequential observations, 
where $x_t \in \mathbb{R}^M$ represents the multivariate measurement at time step $t$. 
Here, $M$ denotes the number of time series channels. In forecasting tasks, the model predicts future values based on past observations, so the input itself serves as the target label for prediction.

In the training/finetuning of TSFMs, it is a common practice to partition the original time series dataset $\mathcal{D}$ into a collection of subsequences 
$\{ S_i \}_{i=1}^{N}$, where each $S_i = \{ x_t \}_{t=t_i}^{t_i+L_s-1}$ contains $L_s$ consecutive time points. Our goal is to estimate a data quality value $v(S_i)$ for each target time series sample $S_i$, reflecting its effect on the performance of TSFMs when employed.

\partitle{Block-Level Time Series Data Value Estimation}
Equipped with theoretical support for the approximation capability of in-context finetuning to influence functions, we describe the detailed algorithm for tailoring in-context finetuning to time series data valuation. 
Specifically, we further divide the target time series dataset $\mathcal{D}$ into overlapping blocks 
$\{ B_{k} \}_{k=1}^{K}$ of fixed length $L$, preserving local temporal dependencies~\cite{hall1985resampling,buhlmann2002bootstraps}.
\begin{equation}
B_{k} = \{ x_{k}, x_{k+1}, \dots, x_{k+L-1} \}.
\end{equation}
Each block $B_{k}$ is treated as a data valuation unit. We apply in-context finetuning to TSFMs (e.g., Time-MoE~\cite{xiaoming2025time}, a 2.4B-parameter model) by treating $B_{k}$ as the target sample and $\mathcal{D}_{\text{context}}$ as the context samples, thereby approximating the influence function value of $B_{k}$ by calculating the context loss following Eq.(\ref{eq:loss-diff}):
\begin{equation}
\text{Infl}(B_{k}, \mathcal{D}_{\text{context}}) 
\propto 
\mathcal{L}(\mathcal{D}_{\text{context}}; \theta) 
- \mathcal{L}(\mathcal{D}_{\text{context}}; \theta_{\mathrm{finetuned}}(B_{k})),
\end{equation}
where $\theta_{\mathrm{finetuned}}(B_{k}) = \theta - \eta \nabla_\theta \mathcal{L}(B_{k}; \theta)$ 
denotes the model parameters after in-context finetuning on block $B_{k}$ with learning rate $\eta$.
This loss difference quantifies how much the model improves (or degrades) after exploiting $B_{k}$,
serving as the block-level time series data value:
\begin{equation}
\label{eq:blockvalue}
v(B_{k}) = \mathcal{L}(\mathcal{D}_{\text{context}}; \theta) 
- \mathcal{L}(\mathcal{D}_{\text{context}}; \theta_{\mathrm{finetuned}}(B_{k})).
\end{equation}

In addition, for multivariate time series with $x_t \in \mathbb{R}^M$, our proposed LTSV approach can be directly applied by simply feeding in-context finetuning with multivariate input blocks. This design avoids the need to either construct a high-dimensional parameter matrix or forcibly split channels for separate processing, as is required in traditional influence function methods.

\partitle{Hierarchical Aggregation from Blocks to Samples} With each time point covered by multiple overlapping blocks, we first compute the time point–level data value by averaging the block-level data values over all blocks that contain that point:
\begin{equation}
v(x_t) = \frac{1}{|\mathcal{B}(x_t)|} 
\sum_{B_{k} \in \mathcal{B}(x_t)} v(B_{k}),
\end{equation}
where $\mathcal{B}(x_t)$ denotes the set of blocks covering the time point $x_t$. Finally, we aggregate all time point-level values to obtain a sample-level data value:
\begin{equation}
\label{eq:samplevalue}
v(S_i) = \frac{1}{L_s} \sum_{t=1}^{L_s} v(x_t),
\end{equation}
This hierarchical valuation process preserves both temporal locality and channel structure, yielding a unified scalar measure that quantifies each sample’s contribution to improving the model’s performance.

\subsection{Computational Complexity Analysis}
\label{sec:efficiency}

Traditional influence function requires the inversion of the Hessian matrix $H_{\theta_t}$, 
which incurs a computational cost of $\mathcal{O}(nP^2 + P^3)$ for $n$ target samples and a model with $P$ parameters. More efficient implementations of the influence function adopt a series of numerical techniques, such as matrix-vector products and conjugate Newton methods, to alleviate the computation of the Hessian and its inverse. However, even with such numerical techniques, influence function calculations still incur prohibitive computational costs for TSFMs with millions to billions of parameters.

Alternatively, LTSV circumvents the Hessian-related calculations by in-context finetuning, which suffices to calculate the first-order gradient, thereby 
reducing the complexity to $\mathcal{O}(nP)$. 
This linear scaling arises because only a single gradient computation is required for each block-level data valuation step. 
Moreover, since blocks are processed independently, all computations can be further parallelized across devices.

In summary, LTSV approximates the influence function backed by theoretical evidence in Theorem~\ref{thm:incontext},
while achieving substantial efficiency gains, making it scalable to TSFMs and long time series datasets. The empirical runtime results validating this advantage are presented in Section~\ref{subsec:efficiency}.

\section{Experiments}
\label{sec:exp}

\subsection{Experimental Setup}
\label{subsec:setup}

\partitle{Datasets, TSFM Types and Baselines}
We evaluate our proposed LTSV on five widely used time series forecasting benchmarks that cover diverse temporal patterns and domains: \textbf{Electricity}~\cite{trindade2015electricityloaddiagrams20112014}, \textbf{Exchange Rate}~\cite{lai2018modeling}, \textbf{Weather}~\cite{weather_data}, \textbf{ETT}~\cite{zhou2021informer}(ETT-m2), and \textbf{Illness}~\cite{illness_data}.
Experiments are conducted on three representative time series foundation models that cover different architectural paradigms of TSFMs: \textbf{Time-MoE}~\cite{woo2024unified}, a mixture-of-experts decoding-only Transformer model; \textbf{Time-LLM}~\cite{liutimer2024timer}, which integrates large language model architectures; and \textbf{MOMENT}~\cite{goswami2024moment}, a T5-based encoder-decoder Transformer model. We compare LTSV against representative data valuation methods, including \textbf{Influence Function}~\cite{zhang2024timeinf}, and \textbf{Data Shapley}~\cite{bento2021timeshap}.

\partitle{Data Valuation Procedure}
We estimate time series data values under the setting of the time series forecasting task, which is not only a fundamental and prevalent task in the time series domain but also particularly well-suited for foundation models. For each dataset, we reserve approximately 30\% of the observations as a hold-out test set, which is strictly excluded from all data valuation computations to ensure unbiased evaluation.  
The remaining 70\% of the data constitute the target set used for data valuation. We then segment the target sequences into fixed-length temporal blocks of size $L = 100$, each treated as a unit for data valuation.  

To estimate the quality value of each sample, we adopt a five-fold partitioning strategy over the target set. In each fold, one partition is used as the target valuation subset, while the remaining four serve as the context subset. For each sample in the target valuation subset, we perform one-step in-context finetuning of the foundation model while computing the change in reference loss. The reference loss, measured by mean squared error (MSE), is calculated using the reference context both before and after finetuning. This change is then used to calculate the sample's valuation score according to Eq.(\ref{eq:blockvalue}) and Eq.(\ref{eq:samplevalue}). During in-context finetuning, we use consistent hyperparameter settings across all experiments, including learning rate $1\times10^{-5}$, Adam optimizer $(\beta_1{=}0.9, \beta_2{=}0.999, \epsilon{=}1\times10^{-8})$, batch size 1, fp32 precision, gradient clipping at 1.0, weight decay of 0.1.

\partitle{Evaluation Protocol}
To assess the effectiveness of the estimated valuations, we perform data selection based on the computed scores in the forecasting setting.  
Specifically, we compare three selection strategies:  
(1) \textbf{Top-$k$} selection, retaining the highest-valued samples;  
(2) \textbf{Bottom-$k$} selection, retaining the lowest-valued samples; and  
(3) \textbf{Full-data} fine-tuning, where all samples are used without selection, serving as a baseline.  
For each strategy, the subset of data is used to fine-tune the corresponding foundation model (ensuring that the model has not been exposed to this dataset beforehand).  
After fine-tuning, we evaluate the model on the held-out test set, using MSE and MAE (mean absolute error) as the primary metrics.
Performance differences across selection strategies reflect the quality and reliability of the proposed valuation mechanism.

\subsection{Main Experiment Results}

In addition to the three data selection strategies described in Section~\ref{subsec:setup}, we include the model performance before any fine-tuning as a reference. For both Top-$k$ and Bottom-$k$ selections, we set the retention ratio to 50\% of the target samples. This allows a direct comparison of model gains when using high-valued versus low-valued data according to LTSV scores.

As shown in Table~\ref{tab:main_results}, experimental results across multiple datasets and foundation models reveal a consistent trend: fine-tuning with the top 50\% of high-valued samples noticeably outperforms fine-tuning with the bottom 50\% of samples, often by a substantial margin. In certain scenarios, the top 50\% even achieves comparable or better results than full-data fine-tuning, indicating that carefully selected high-quality samples can be as effective as using the entire target set. Fine-tuning on the bottom 50\% of samples results in only modest or negligible improvements compared to the initial model, indicating that LTSV can distinguish less informative data from more impactful samples.

These results demonstrate two key points: (1) the valuation scores accurately reflect the contribution of each sample to model performance, and (2) selecting high-valued data based on LTSV provides a practical way to improve time series forecasting models, even with a reduced dataset.

\begin{table*}[t]
\caption{Main results on five forecasting datasets with three time series foundation models. Each cell reports MSE and MAE results. Fine-tuning with the Top-50\% of high-valued samples consistently outperforms the Bottom-50\% and often approaches or exceeds full-data fine-tuning. \textbf{Note:} For the \emph{Top} sample selection strategy, lower MSE/MAE indicates better performance; 
for the \emph{Bottom} strategy, higher MSE/MAE indicates better performance.}
\label{tab:main_results}
\centering
\footnotesize
\setlength{\tabcolsep}{3.5pt}
\begin{tabular}{l|c c|c c|c c|c c|c c}
\toprule
\textbf{Strategies} & 
\multicolumn{2}{c|}{\textbf{Electricity}} &
\multicolumn{2}{c|}{\textbf{Ex.Rate}} &
\multicolumn{2}{c|}{\textbf{Weather}} &
\multicolumn{2}{c|}{\textbf{Illness}} &
\multicolumn{2}{c}{\textbf{ETT}} \\
 & MSE & MAE & MSE & MAE & MSE & MAE & MSE & MAE & MSE & MAE \\
\midrule
\multicolumn{11}{c}{\textbf{Time-MoE}} \\
\midrule
Pre-trained          & 3.858 & 0.015 & 7.776 & 6.3e-5 & 16.94 & 3.252 & 0.924 & 0.707 & 0.126 & 3.50e-3 \\
Full-data            & 0.978 & 0.008 & 4.579 & 5.3e-5 & 16.81 & 3.006 & 0.905 & 0.772 & 0.123 & 3.46e-3 \\
\rowcolor{magenta!10} Bottom-50\%   & 1.276 & 0.009 & 5.665 & 5.7e-5 & 16.69 & 2.975 & 0.863 & 0.741 & 0.125 & 3.48e-3 \\
\rowcolor{cyan!10} Top-50\%     & 0.851 & 0.007 & 4.335 & 5.3e-5 & 15.23 & 2.919 & 0.851 & 0.728 & 0.124 & 3.48e-3 \\
\midrule
\multicolumn{11}{c}{\textbf{Time-LLM}} \\
\midrule
Pre-trained          & 3.143 & 1.541 & 2.862 & 1.488 & 0.804 & 0.743 & 1.492 & 0.985 & 0.831 & 0.698 \\
Full-data            & 2.472 & 1.407 & 0.445 & 0.593 & 0.690 & 0.722 & 1.439 & 0.932 & 0.703 & 0.601 \\
\rowcolor{magenta!10} Bottom-50\%   & 2.659 & 1.442 & 0.439 & 0.586 & 0.734 & 0.745 & 1.430 & 0.949 & 0.745 & 0.651 \\
\rowcolor{cyan!10} Top-50\%     & 2.546 & 1.401 & 0.436 & 0.585 & 0.697 & 0.724 & 1.403 & 0.931 & 0.724 & 0.605 \\
\midrule
\multicolumn{11}{c}{\textbf{MOMENT}} \\
\midrule
Pre-trained          & 0.130 & 0.326 & 0.109 & 0.283 & 0.107 & 0.278 & 0.221 & 0.191 & 0.127 & 0.301 \\
Full-data            & 0.106 & 0.281 & 0.097 & 0.263 & 0.079 & 0.235 & 0.182 & 0.157 & 0.106 & 0.249 \\
\rowcolor{magenta!10} Bottom-50\%   & 0.132 & 0.320 & 0.110 & 0.286 & 0.117 & 0.293 & 0.190 & 0.167 & 0.124 & 0.283 \\
\rowcolor{cyan!10} Top-50\%     & 0.098 & 0.268 & 0.102 & 0.273 & 0.074 & 0.215 & 0.180 & 0.152 & 0.119 & 0.286 \\
\bottomrule
\end{tabular}
\vspace{-1em}
\end{table*}

\subsection{Computational Efficiency}
\label{subsec:efficiency}

We further evaluate the computational efficiency of LTSV by comparing it with the classical influence function method on the Illness dataset under identical experimental settings, 
with all experiments conducted on the same hardware environment equipped with an NVIDIA A100 GPU.
The only difference lies in the computation mechanism: the influence function requires Hessian inversion, while LTSV performs lightweight in-context fine-tuning for each sample.

Fig.~\ref{fig:efficiency} presents the total valuation time (in seconds) versus model parameter size across representative time series models, including Linear (300), LSTM (4K), PatchTST (300K), MOMENT-base (40M), and TimeMoE-large (200M).  
The empirical results are consistent with theoretical complexity analysis. Specifically, the cost of influence-function–based valuation increases rapidly with model size, exhibiting near-cubic scaling due to second-order derivative computations. For the two foundation models, executing the full influence function becomes prohibitively time-consuming; therefore, we estimate its computational cost based on a subset of samples to approximate the total runtime.
In contrast, LTSV grows almost linearly, as it only requires a single forward–backward pass per sample during fine-tuning.

These results confirm that LTSV achieves both high efficiency and scalability, maintaining linear complexity while supporting modern foundation models with hundreds of millions of parameters. 

\begin{figure}[t]
    \centering
    \includegraphics[width=0.6\textwidth]{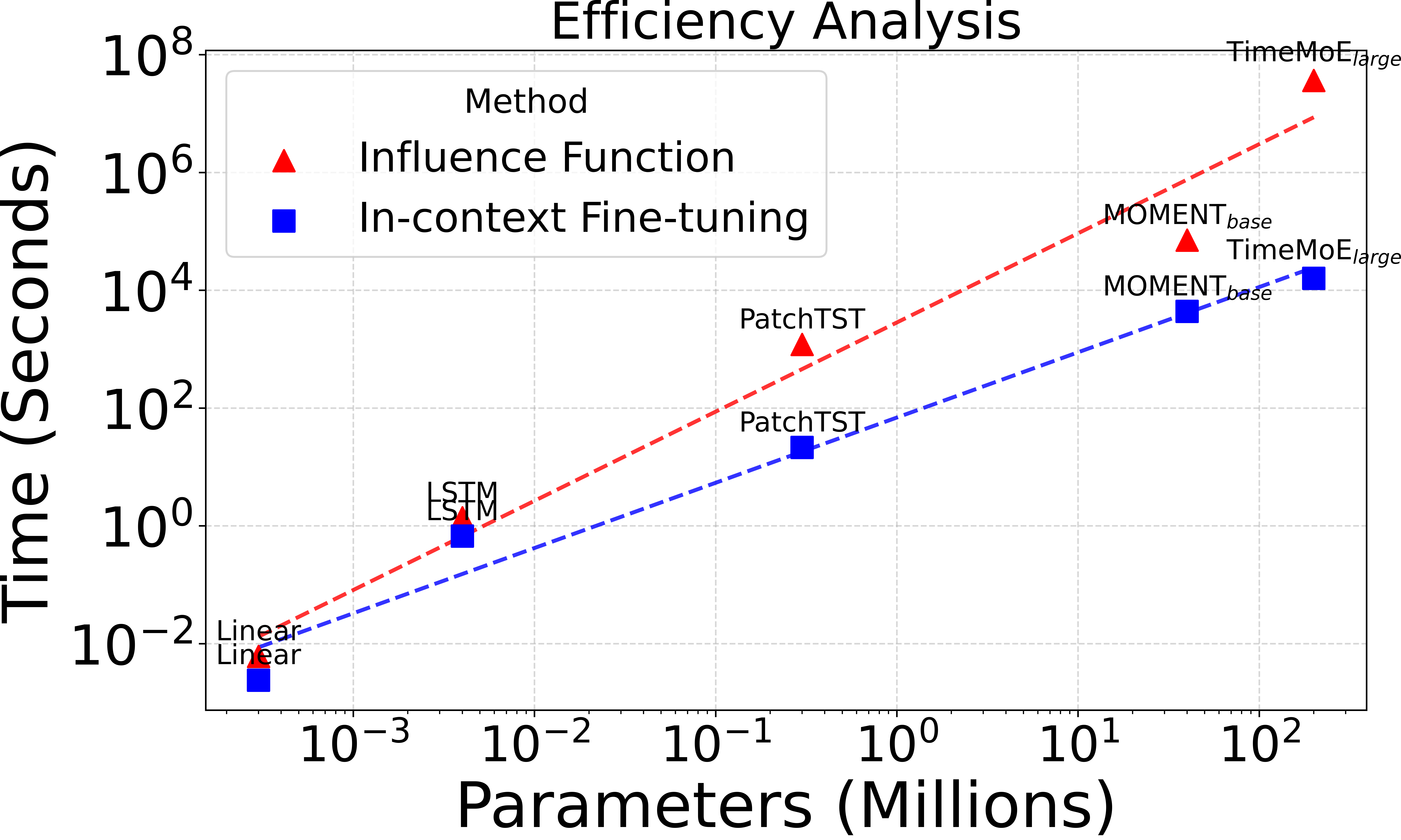}
    \vskip -0.5em
    \caption{Computational efficiency comparison between LTSV and influence function across models of varying parameter sizes.}
    \label{fig:efficiency}
    \vspace{-1em}
\end{figure}

\subsection{Generalization of Foundation Model Valuations to Diverse Downstream Models}

In this section, we investigate whether the data valuations obtained from the foundation model can generalize effectively to diverse downstream time series models. Specifically, we train each downstream model using five subsets of data: (1) the \textbf{top-valued} samples according to LTSV scores, (2) the \textbf{bottom-valued} samples from LTSV, (3) the samples ranked highest by the classical \textbf{influence function} adapted for time series~\cite{zhang2024timeinf}, (4) the samples ranked highest by the \textbf{Shapley value} method (time series version)~\cite{bento2021timeshap}, and (5) \textbf{randomly} selected samples. The selection ratio is fixed at 20\% to better highlight the performance differences among methods.
The data valuations are obtained from the \textbf{Time-MoE} foundation model. The downstream forecasting models cover a wide range of architectures, including \textbf{DLinear}~\cite{zeng2023transformers}, \textbf{PatchTST}~\cite{nie2022time}, and \textbf{PAttn}~\cite{tan2024language}, representing linear, transformer-based, and attention-based paradigms respectively.

\begin{table}[t]
\centering
\footnotesize
\setlength{\tabcolsep}{3.6pt}
\caption{Generalization results of foundation-model-based data valuations across three downstream time series models. Valuations obtained from in-context fine-tuned foundation models show strong transferability and remain highly predictive of data quality across different architectures. \textbf{Note:} For the \emph{Top} sample selection strategy, lower MSE/MAE indicates better performance; for the \emph{Bottom} strategy, higher MSE/MAE indicates better performance.}
\label{tab:results_generalization}
\begin{tabular}{l|c c|c c|c c|c c|c c}
\toprule
\textbf{Methods} & 
\multicolumn{2}{c|}{\textbf{Electricity}} &
\multicolumn{2}{c|}{\textbf{Ex.Rate}} &
\multicolumn{2}{c|}{\textbf{Weather}} &
\multicolumn{2}{c|}{\textbf{Illness}} &
\multicolumn{2}{c}{\textbf{ETT}} \\
 & MSE & MAE & MSE & MAE & MSE & MAE & MSE & MAE & MSE & MAE \\
\midrule
\multicolumn{11}{c}{\textbf{DLinear}} \\
\midrule
Random           & 2.221 & 1.353 & 0.406 & 0.562 & 0.370 & 0.514 & 1.982 & 1.259 & 0.372 & 0.506 \\
Influence~\cite{zhang2024timeinf}        & 2.074 & 1.310 & 0.360 & 0.522 & 0.322 & 0.478 & 1.718 & 1.142 & 0.355 & 0.496 \\
Shapley~\cite{bento2021timeshap}          & 2.047 & 1.301 & 0.372 & 0.529 & 0.384 & 0.520 & 1.603 & 1.102 & 0.325 & 0.475 \\
\rowcolor{magenta!10} Bottom(ours)  & 2.487 & 1.435 & 0.440 & 0.580 & 0.481 & 0.580 & 2.895 & 1.496 & 0.486 & 0.579 \\
\rowcolor{cyan!10} Top(ours)     & 2.060 & 1.304 & 0.387 & 0.544 & 0.311 & 0.468 & 1.668 & 1.130 & 0.301 & 0.459 \\
\midrule
\multicolumn{11}{c}{\textbf{PatchTST}} \\
\midrule
Random           & 1.930 & 1.206 & 0.336 & 0.496 & 0.363 & 0.502 & 1.530 & 1.050 & 0.385 & 0.465 \\
Influence~\cite{zhang2024timeinf}        & 1.830 & 1.131 & 0.286 & 0.456 & 0.355 & 0.508 & 1.099 & 0.861 & 0.314 & 0.459 \\
Shapley~\cite{bento2021timeshap}          & 1.888 & 1.145 & 0.288 & 0.460 & 0.366 & 0.516 & 1.057 & 0.844 & 0.266 & 0.411 \\
\rowcolor{magenta!10} Bottom(ours)  & 2.811 & 1.470 & 0.380 & 0.522 & 0.370 & 0.524 & 1.958 & 1.214 & 0.421 & 0.512 \\
\rowcolor{cyan!10} Top(ours)     & 1.906 & 1.187 & 0.317 & 0.475 & 0.360 & 0.498 & 1.103 & 0.873 & 0.288 & 0.421 \\
\midrule
\multicolumn{11}{c}{\textbf{PAttn}} \\
\midrule
Random           & 2.250 & 1.245 & 0.356 & 0.494 & 0.324 & 0.474 & 1.350 & 1.057 & 0.332 & 0.441 \\
Influence~\cite{zhang2024timeinf}        & 1.883 & 1.139 & 0.308 & 0.462 & 0.326 & 0.487 & 1.025 & 0.841 & 0.326 & 0.468 \\
Shapley~\cite{bento2021timeshap}          & 1.925 & 1.153 & 0.313 & 0.478 & 0.474 & 0.574 & 0.965 & 0.806 & 0.292 & 0.432 \\
\rowcolor{magenta!10} Bottom(ours)  & 2.757 & 1.456 & 0.389 & 0.528 & 0.440 & 0.539 & 1.974 & 1.201 & 0.431 & 0.520 \\
\rowcolor{cyan!10} Top(ours)     & 1.911 & 1.185 & 0.332 & 0.487 & 0.278 & 0.428 & 1.036 & 0.858 & 0.280 & 0.418 \\
\bottomrule
\end{tabular}
\vspace{-1em}
\end{table}

\begin{figure}[!h]
    \centering
    \begin{minipage}[b]{0.45\textwidth}
        \centering
        \includegraphics[width=1.0\textwidth]{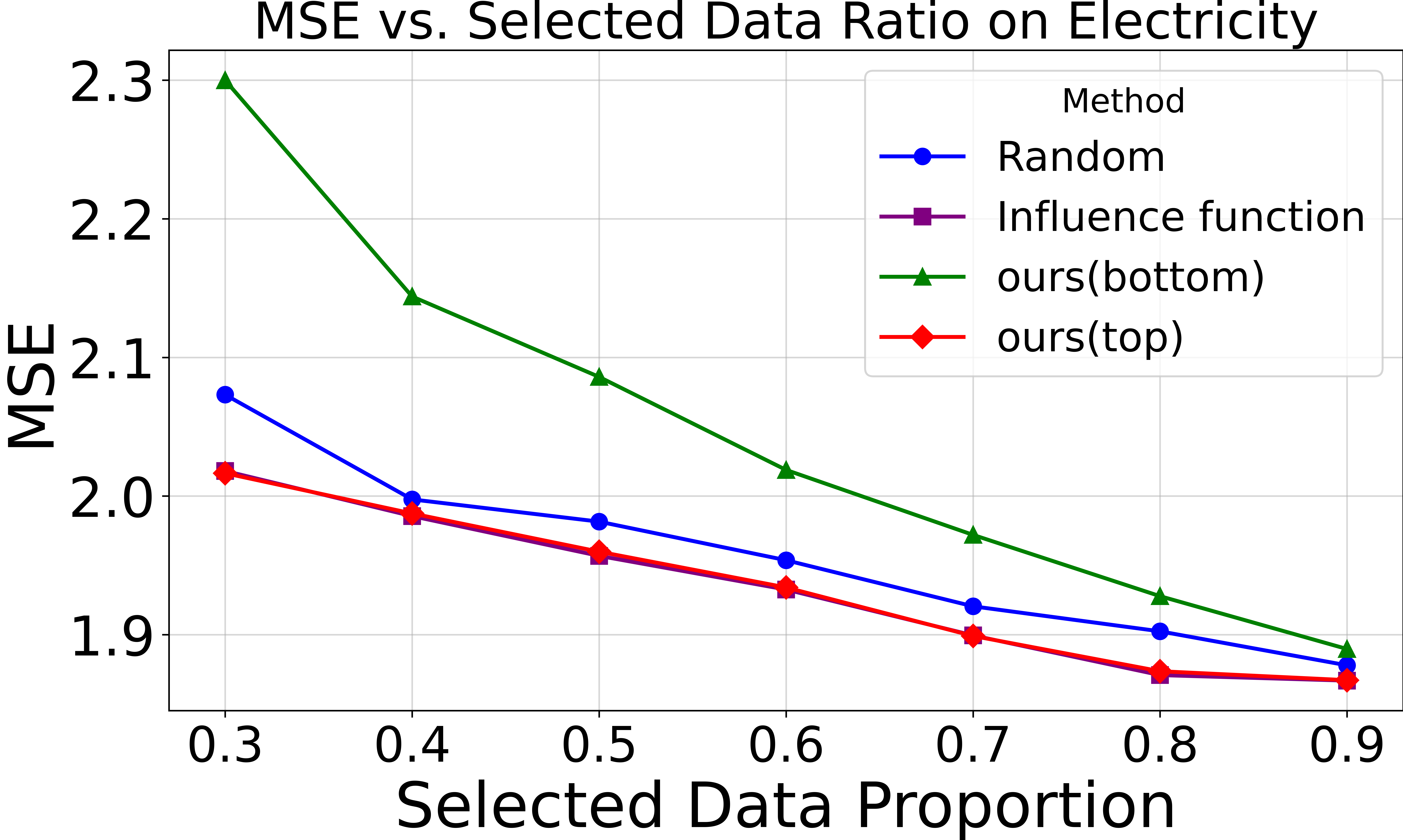}
    \end{minipage} \hfill
    \begin{minipage}[b]{0.45\textwidth}
        \centering
        \includegraphics[width=1.0\textwidth]{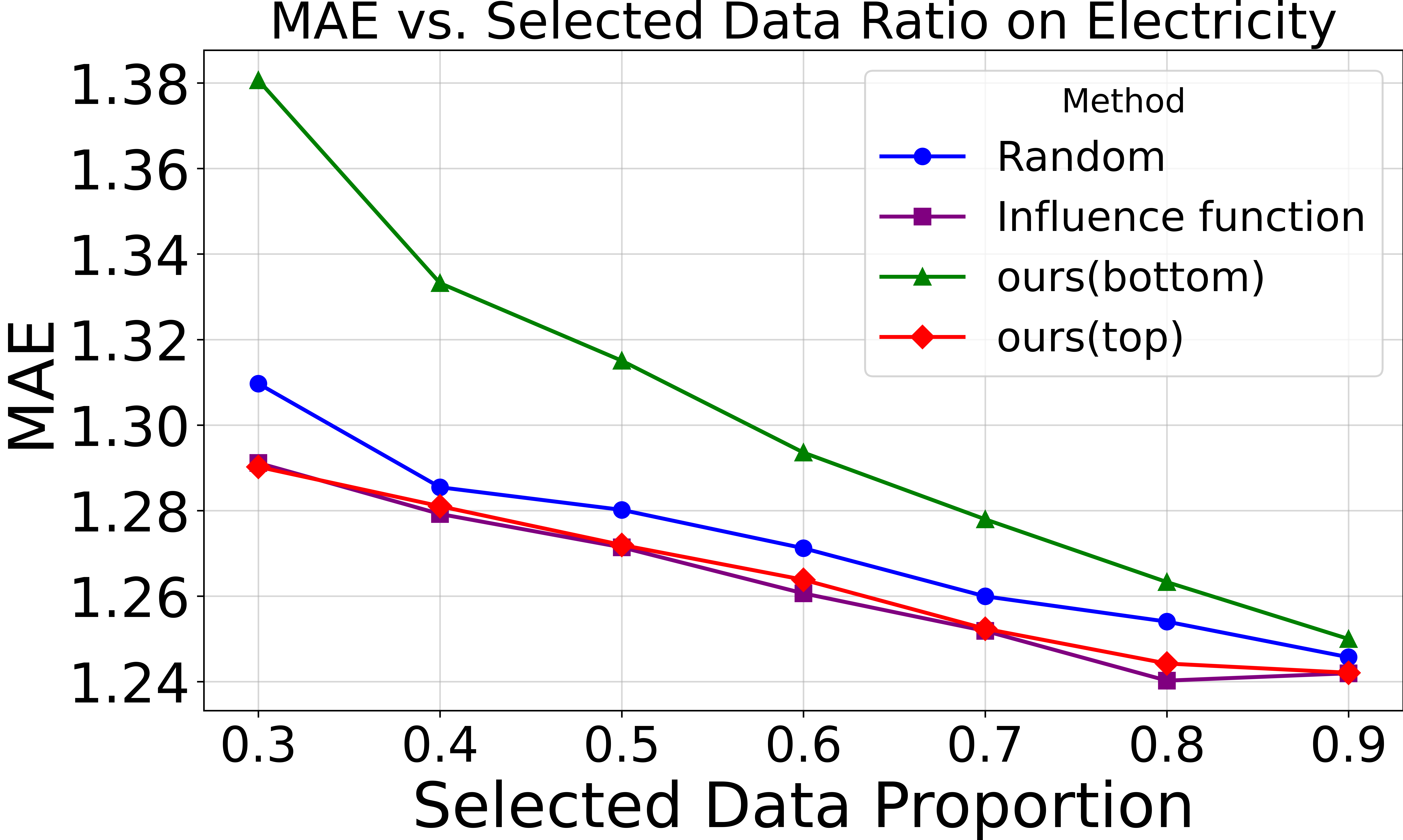}
    \end{minipage}
    \vskip -0.8em
    \caption{
    Performance variation with different selection ratios on the Electricity dataset under DLinear. 
    The curves show a consistent trend where bottom-valued samples yield higher MSE and MAE than random selection, while top-valued samples outperform both. 
    As the selection ratio increases, the performance gap among the three gradually narrows. 
    Notably, the curve of LTSV (top) remains close to that of the Influence Function, indicating strong consistency with classical valuation results.
    }
    \label{fig:selectionratio}
    \vspace{-1.5em}
\end{figure}

The results presented in Table~\ref{tab:results_generalization} demonstrate that the data valuations obtained through in-context fine-tuning on the foundation model can be effectively transferred to diverse downstream architectures. 
Across all architectures, the models trained on high-valued data consistently outperform those trained on low-valued subsets, while the random selection results generally lie in between. 
This indicates that downstream models remain sensitive to the quality of the selected data, and the valuation scores learned from the foundation model can still capture meaningful distinctions in data utility.
Moreover, the performance achieved using foundation-model valuations closely approaches that obtained by applying influence function or Shapley value estimations computed directly on each downstream model, which theoretically yield more faithful estimates of individual sample contributions.
These results highlight the strong cross-model generalization of LTSV, confirming that valuations derived from in-context fine-tuning on foundation models can serve as transferable and reliable indicators of data quality for downstream time series learning. We further analyze the effect of varying selection ratios on the Electricity dataset under the DLinear model, as shown in Fig.~\ref{fig:selectionratio}.

In addition to Time-MoE, we further conduct experiments on two other foundation models to verify whether the valuations also generalize effectively across different downstream models. The results are summarized in Fig.~\ref{fig:othertsfm}, which compares the performance of downstream models trained on top- and bottom-valued samples.

\subsection{Ablation Study on Block Length}

In the main experiments, the temporal block length is fixed at 100.  
To investigate the impact of block length on both data valuation and forecasting performance, we repeat the experiments under multiple block lengths, specifically $L \in \{50, 75, 100, 125\}$.  
The experiments are conducted on the Electricity dataset using the DLinear model, with a data selection ratio of 20\%.

As shown in Table~\ref{tab:ablation_window}, varying the block length between 50 and 125 has only a modest impact on the effectiveness of data valuation.
Across all lengths, fine-tuning on the top 20\% of high-valued samples consistently outperforms the bottom 20\%, clearly demonstrating the reliability of LTSV scores. 
Interestingly, when using block lengths of $75$ or $100$, the Top selection slightly surpasses the Influence function baseline, suggesting that a moderate block size may improve the stability and fidelity of data valuation.  
Overall, these results indicate that LTSV is robust to reasonable variations in temporal block length, while moderate blocks provide optimal performance.

\begin{figure}[t]
    \centering
    \begin{minipage}[b]{0.45\textwidth}
        \centering
        \includegraphics[width=1.0\textwidth]{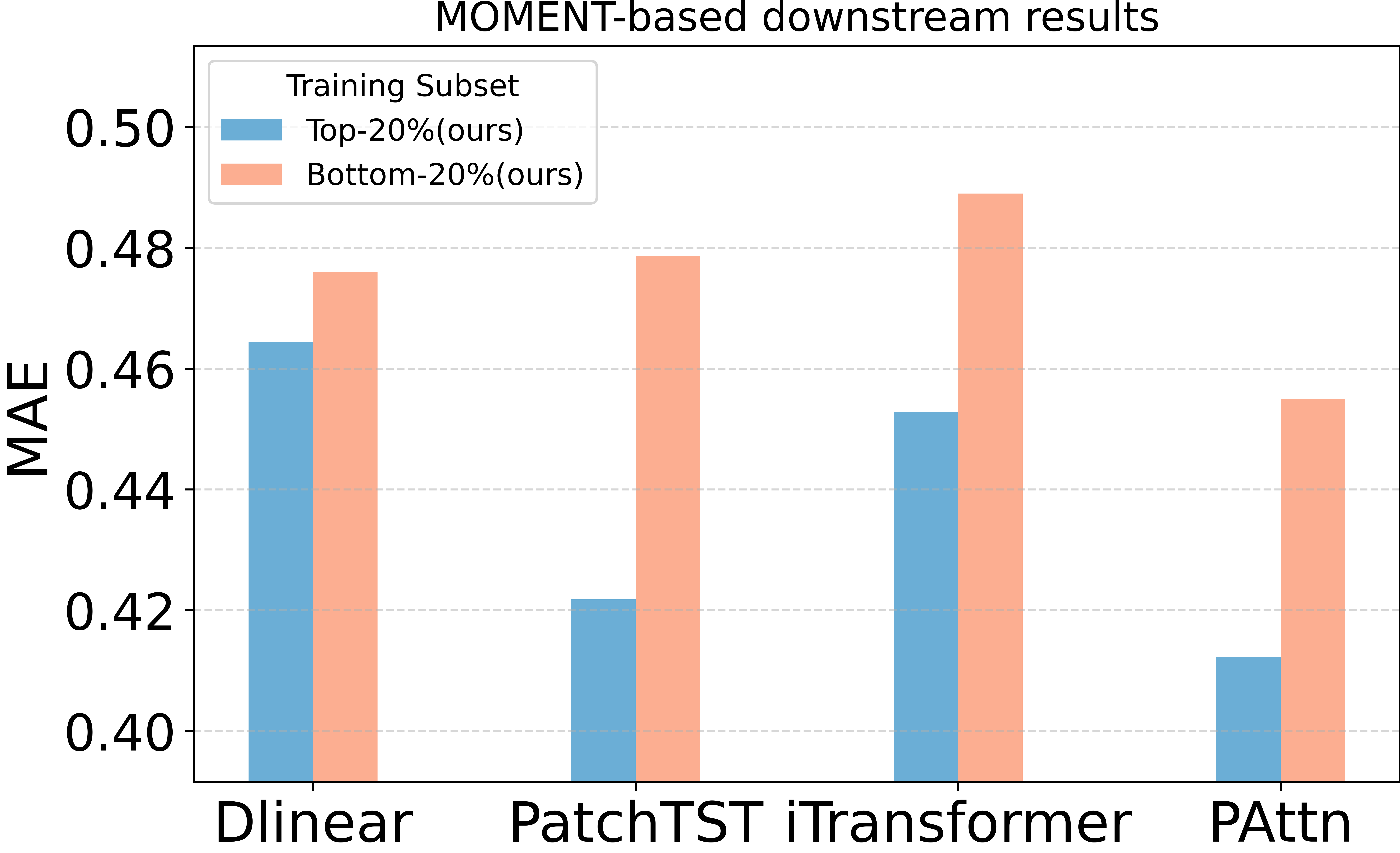}
    \end{minipage} \hfill
    \begin{minipage}[b]{0.45\textwidth}
        \centering
        \includegraphics[width=1.0\textwidth]{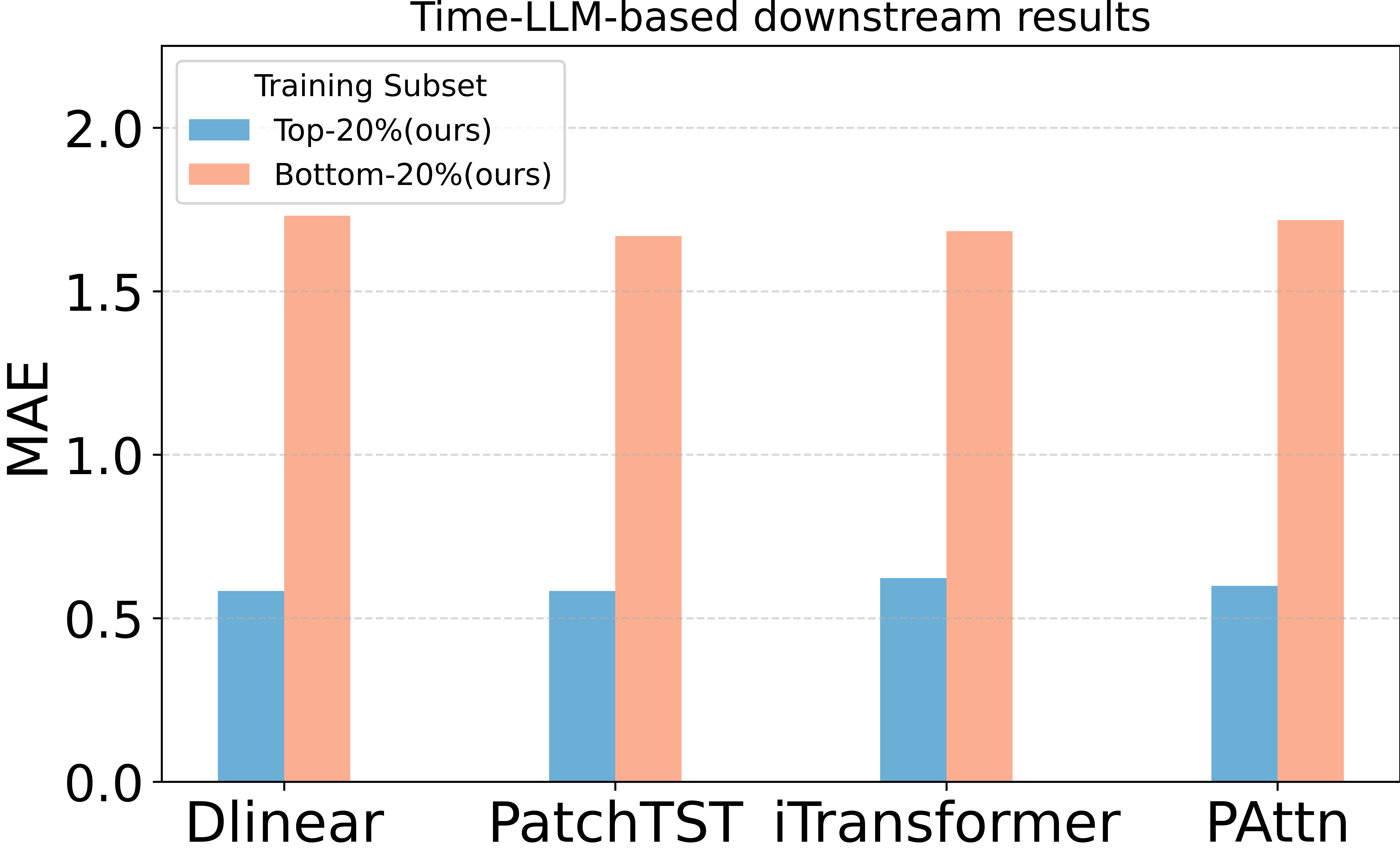}
    \end{minipage}
    \vskip -0.8em
    \caption{
    Downstream performance when trained on top- and bottom-valued samples identified by LTSV. 
    The left panel shows results based on the \textbf{MOMENT} foundation model, while the right panel corresponds to \textbf{Time-LLM}.
    Models trained on bottom-valued samples show consistently higher MAE, while top-valued samples yield much lower errors, confirming the robustness and transferability of LTSV across architectures.
    }
    \label{fig:othertsfm}
    \vspace{-1.5em}
\end{figure}
\begin{table}[!h]
\centering
\footnotesize
\setlength{\tabcolsep}{4pt}
\caption{Ablation results on the effect of temporal block length on data valuation and forecasting performance. Experiments are conducted on the Electricity dataset under the DLinear model. \textbf{Note:} For the \emph{Top} sample selection strategy, lower MSE/MAE indicates better performance; for the \emph{Bottom} strategy, higher MSE/MAE indicates better performance.}
\label{tab:ablation_window}
\begin{tabular}{l|c c|c c|c c|c c}
\toprule
\textbf{Method / Block} & \multicolumn{2}{c|}{$50$} & \multicolumn{2}{c|}{$75$} & \multicolumn{2}{c|}{$100$} & \multicolumn{2}{c}{$125$} \\
 & MSE & MAE & MSE & MAE & MSE & MAE & MSE & MAE \\
\midrule
Random           & 1.336 & 1.015 & 1.831 & 1.213 & 2.221 & 1.353 & 2.043 & 1.285 \\
Influence function~\cite{zhang2024timeinf}  & 1.281 & 0.994 & 1.678 & 1.164 & 2.074 & 1.310 & 2.017 & 1.277 \\
\rowcolor{magenta!10} Bottom (ours)      & 1.381 & 1.037 & 1.722 & 1.173 & 2.487 & 1.435 & 2.078 & 1.297 \\
\rowcolor{cyan!10} Top (ours)         & 1.307 & 1.007 & 1.668 & 1.157 & 2.060 & 1.304 & 2.017 & 1.277 \\
\bottomrule
\end{tabular}
\vspace{-1em}
\end{table}

\section{Conclusion}
\label{sec:conclusion}

In this paper, we propose LTSV, a lightweight framework for time series data valuation on foundation models via in-context fine-tuning. 
It efficiently approximates influence scores without computing Hessian-related terms and captures temporal dependencies through temporal block aggregation. 
Comprehensive experiments across five datasets and three model architectures demonstrate that LTSV provides reliable, scalable, and transferable valuations for capturing meaningful sample contributions. 

\begin{credits}
\subsubsection{\ackname}
    We would like to thank the three anonymous reviewers, the area chairs, and the program chairs for their constructive comments and efforts. 
\end{credits}

{
\small
\bibliography{dasfaa}
\bibliographystyle{splncs04}
}

\end{document}